\documentclass{article}

\usepackage{microtype}
\usepackage{graphicx}
\usepackage{booktabs}
\usepackage{hyperref}

\usepackage[accepted]{icml2026}

\usepackage{amsmath}
\usepackage{amssymb}
\usepackage{amsthm}
\usepackage{pifont}
\usepackage{multirow}
\usepackage{xcolor}
\usepackage{algorithm}
\usepackage{algorithmic}
\newcommand{\RETURN}{\STATE \textbf{return} }
\usepackage{dblfloatfix}

\newtheorem{theorem}{Theorem}
\newtheorem{proposition}{Proposition}

\newcommand{\cmark}{\ding{51}}
\newcommand{\xmark}{\ding{55}}
\newcommand{\pmark}{\ding{51}\textsuperscript{*}}

\theoremstyle{definition}

\icmltitlerunning{The Containment Gap in Deployed Agentic AI Frameworks}

\begin{document}

\twocolumn[
\icmltitle{The Containment Gap: How Deployed Agentic AI Frameworks\\Fail Public-Facing Safety Requirements}

\icmlsetsymbol{equal}{*}

\begin{icmlauthorlist}
\icmlauthor{Md Jafrin Hossain}{fiu}
\icmlauthor{Mohammad Arif Hossain}{fiu}
\icmlauthor{Weiqi Liu}{njit}
\icmlauthor{Nirwan Ansari}{njit}
\end{icmlauthorlist}

\icmlaffiliation{fiu}{Department of Electrical and Computer Engineering, Florida International University, Miami, FL, USA}
\icmlaffiliation{njit}{Department of Electrical and Computer Engineering, New Jersey Institute of Technology, Newark, NJ, USA}

\icmlcorrespondingauthor{Md Jafrin Hossain}{mhoss078@fiu.edu}

\icmlkeywords{Agentic AI, LLM Security, Trustworthy AI, Framework Audit, Containment}

\vskip 0.3in
]

\printAffiliationsAndNotice{}

\begin{abstract}
Agentic large language model systems that autonomously invoke tools, maintain persistent memory, and execute multi-step plans are increasingly deployed in public-facing domains, including government services, healthcare triage, and financial advising. We ask whether the frameworks used to build these systems provide architectural-level structural safety guarantees. Applying six containment principles derived from a compositional model of agentic architectures, we audit three dominant frameworks (LangChain, AutoGPT, and OpenAI Agents SDK) and find no native compliance in any of them. Memory integrity, a defense against one of the most prevalent vulnerability classes, is not observed in any of the three evaluated frameworks. We validate these findings empirically: in a simulated government benefits agent built on LangChain, a single memory-poisoning write induces persistent targeted corruption across all tested seeds and backends, increasing the wrongful denial rate for targeted applicants to 88.9\%. Under a complex five-factor policy, the same attack preserves aggregate accuracy while increasing targeted wrongful denials by $3.5\times$, rendering the corruption difficult to detect through standard monitoring. We then introduce two lightweight containment mechanisms: a memory integrity validator and a policy gate, which eliminate both attack vectors with sub-millisecond overhead ($<0.2\,\text{ms}$ per call). We conclude that the current agentic framework ecosystem may not yet meet secure-by-default expectations for public-facing deployments and outline priority architectural interventions to enable trustworthy deployment in high-stakes, socially impactful applications.
\end{abstract}

\section{Introduction}
\label{sec:intro}

Agentic AI systems are increasingly deployed in public-facing domains such as government services, healthcare, and finance \citep{xu2024theagentcompany}. Unlike traditional LLM chatbots, these systems invoke tools, maintain persistent memory, and act autonomously over multi-step horizons \citep{yao2022react}. A single corrupted reasoning cycle can propagate through tool execution into memory, poisoning subsequent interactions and potentially leading to persistent, system-level failures with real-world consequences.

The AI safety community has primarily focused on what models \emph{say}, such as output toxicity, bias, and hallucination, while the trustworthy AI community has emphasized behavioral evaluation and fairness. However, neither has systematically addressed a more fundamental question: \emph{Do the frameworks used to build agentic AI systems provide structural safety guarantees at the architectural level?} This question is orthogonal to model-level safety; it concerns whether the surrounding system enforces reliable boundaries between perception and the core stages (reasoning, execution, and memory) through which every agentic action propagates. While prior work primarily catalogs attack types in LLM agents, it remains unclear why these vulnerabilities consistently persist across different frameworks and model backends. We argue that the root cause is structural: the absence of enforced containment at architectural boundaries.

This paper makes four contributions. First, we present, to the best of our knowledge, the first audit methodology that operationalizes formal containment principles into a reusable compliance matrix for agentic frameworks (Section~\ref{sec:methodology}). Second, auditing LangChain \citep{langchain2024}, AutoGPT \citep{autogpt2024}, and the OpenAI Agents SDK \citep{openai2024agents} suggests that we do not observe native compliance across any of the six principles (Section~\ref{sec:results}). Third, we show that a single memory poisoning write can induce targeted corruption across five backends and, under a five-factor policy, can remain difficult to detect through aggregate metrics (Section~\ref{sec:experiments}). Fourth, two deterministic interventions substantially reduce attack success rates with sub-millisecond overhead (Section~\ref{sec:experiments}).

\section{Background: The Composition Problem}
\label{sec:background}

\subsection{Agentic Systems as Compositional Pipelines}

An agentic LLM system composes four functional stages in a recursive loop \citep{yao2022react, masterman2024landscape}: a \emph{perception} function $P$ that processes external inputs, a \emph{reasoning} (behavior) function $B$ that plans actions using the current input and persistent memory $m_t$, an \emph{execution} function $E$ that invokes tools, and a \emph{memory update} function $\mathcal{U}$ that writes outcomes back to persistent state. The decision-to-action mapping at each timestep is:
\begin{equation}
  \Phi(o_t, m_t) = E\!\left(B\!\left(P(o_t),\, m_t\right)\right),
  \label{eq:phi}
\end{equation}
where $\Phi$ produces the executed action, and the resulting state is subsequently incorporated via $\mathcal{U}$. 

Individual stages may be secure in isolation, but their composition becomes vulnerable without inter-layer isolation \citep{christodorescu2025systems}. Corrupted outputs propagate across stages, from perception to reasoning, execution, and memory, affecting subsequent cycles. Figure~\ref{fig:architecture} illustrates this pipeline with containment gates at each boundary.

\begin{figure}[t]
  \centering
  \includegraphics[width=\columnwidth]{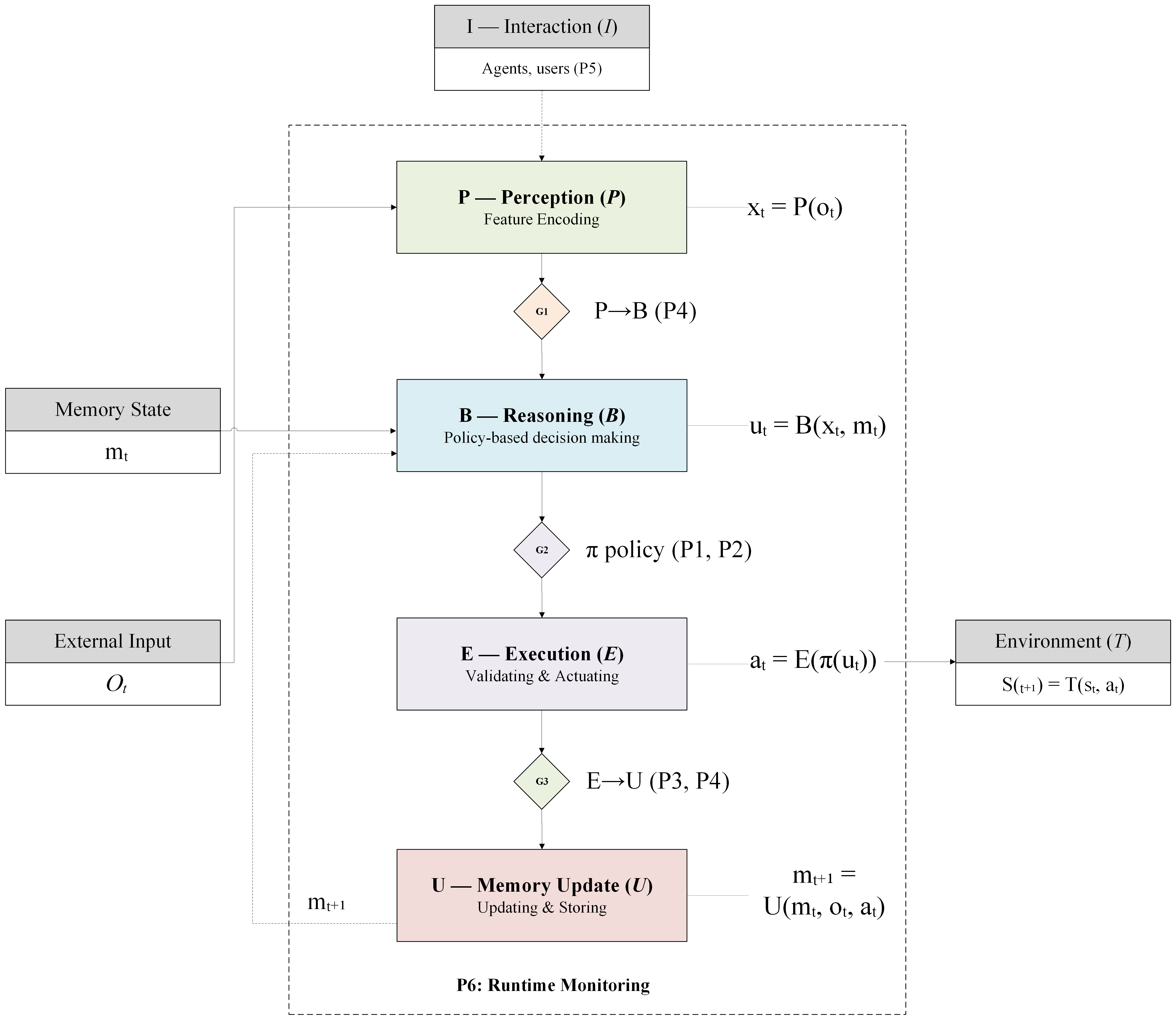}
  \caption{Compositional agentic architecture with containment gates (G1--G3) at layer boundaries. External input $O_t$ and memory state $m_t$ flow through perception, reasoning, execution, and memory update. Gates enforce the six containment principles (P1--P6) at each transition. Runtime monitoring (P6) spans all stages.}
  \label{fig:architecture}
\end{figure}

\subsection{Execution Containment}

Security requires that $\Phi(o_t, m_t) \in \mathcal{C}$ for all $t$, where $\mathcal{C} \subseteq \mathcal{A}$ is a policy-constrained safe action space $\mathcal{A}$ \citep{saltzer1975protection}. We call this condition \emph{execution containment}. When $E$ directly forwards $B$'s output to the runtime without such projection, the system operates in a state of \emph{autonomy without containment}, analogous to executing user-space code with kernel privileges \citep{klein2009sel4}. This parallels the reference monitor concept in systems security and the projection operator in constrained control theory.

\subsection{Six Containment Principles}

We identify six containment factors that ensure safety when applied at layer interfaces. These form the foundation of our audit framework:








\begin{enumerate}
    \item \textbf{Reasoning-Execution Separation (P1):} Policy gates~$\pi$ lie between planning and execution so that the agent cannot implement every plan it devises. The only actions that pass the gate are those that satisfy $E(\pi(u_t)) \in \mathcal{C}$.

    \item \textbf{Capability Scoping (P2):} A bounded token ~ $T_k$ is given to each session that defines which tools can be used, the ranges of parameters, the limits on rates, and expiry times. The agent simply does not have the possibility to break through these boundaries set by the token.

    \item \textbf{Memory Integrity (P3):} The validity of any write before it reaches long-term memory is tested by an integrity function~$\mathcal{I}$. Those writings that do not pass the test are discarded.

    \item \textbf{Layer-Transition Validation (P4):} Security checks are performed at all interfaces that the data traverses, not only the input interface ($P \!\to\! B$, $B \!\to\! E$, $E \!\to\! \mathcal{U}$). Thus, if a malicious user can pass through one interface, they are still not guaranteed to pass through others.

    \item \textbf{Authenticated Communication (P5):} All messages exchanged between agents should contain credentials such as digital signatures that can be verified. Any messages without proper verification credentials are quarantined.

    \item \textbf{Runtime Monitoring (P6):} Lastly, the anomaly detector monitors execution paths as they develop. If it detects an anomaly, it activates containment to mitigate its effects.

\end{enumerate}


We formalize the relationship between these principles and containment below.

\begin{theorem}[Containment Sufficiency]
If an agentic system satisfies P1 (policy-gated execution) and P3 (validated memory writes), then no single-step memory-poisoning attack can induce a persistent policy violation.
\end{theorem}

\begin{proof}[Proof sketch]
P3 guarantees that $\mathcal{I}$ will reject any adversarial write $\delta$, so that $m_{t+1} = m_t$ stays unaltered. P1 guarantees that the expectation $E(\pi(u_t))$ belongs to $\mathcal{C}$. Since memory is safe and execution is controlled, nothing gets propagated along any single-step trajectory that violates safety. Our experiments confirm: with both active, corruption drops from 1.000 to 0.000 across all backends (Table~\ref{tab:multimodel}).
\end{proof}

\begin{proposition}[Need for Combined Enforcement]
Neither P1 nor P3 alone is enough for containment. Dropping P3 while keeping P1 allows memory corruption, leading to future biased inputs into reasoning. Dropping P1 while keeping P3 allows the execution of unsafe actions in one cycle.
\end{proposition}

Experiments~1 and~3 provide empirical evidence: removing memory validation (P3) yields complete corruption (Table~3), while removing the policy gate (P1/P2) yields complete tool bypass (Table~4).

\section{Audit Methodology}
\label{sec:methodology}

\textbf{Framework selection.} 
For our analysis, we will consider the most commonly used deployment systems for agents – LangChain Agents \citep{langchain2024}, AutoGPT \citep{autogpt2024}, and the OpenAI Agents SDK \citep{openai2024agents}.

\textbf{Evidence sources.} The sources used for this study include official documents, source code examination, and security studies that have been published \citep{christodorescu2025systems, ferrag2025prompt}.

\textbf{Scoring rubric.} \cmark~= native default (enabled without configuration); \pmark~= requires explicit configuration; \xmark~= absent. We focus on \emph{default behavior} \citep{christodorescu2025systems, raza2025trism}.

\textbf{Reliability.} Reliability was ensured by having two raters score all 18 framework-principle dyads (Cohen's $\kappa = 0.81$). Conflicts between raters were settled through discussion involving a third rater.

\textbf{Limitations.} The rubric captures mechanism \emph{presence}, not implementation depth or runtime effectiveness. The evaluation is point-in-time, and frameworks may evolve. Runtime testing is presented in Section~\ref{sec:experiments}.

\section{Results: The Compliance Matrix}
\label{sec:results}

Table~\ref{tab:compliance} presents the full compliance matrix. The audit reveals four systemic patterns.

\begin{table*}[t]
\caption{Compliance matrix: production agent frameworks against six containment principles. \cmark{} = native default; \pmark{} = requires configuration; \xmark{} = absent. No framework achieves \cmark{} on any principle.}
\label{tab:compliance}
\small
\begin{tabular*}{\textwidth}{@{\extracolsep{\fill}}lp{4.0cm}p{4.0cm}p{4.0cm}@{}}
\toprule
\textbf{Principle} & \textbf{LangChain Agents} &
  \textbf{AutoGPT} & \textbf{OpenAI Agents SDK} \\
\midrule
P1: Reasoning-Exec.\ Sep.
  & \pmark~Callbacks; not enforced by default
  & \xmark~No policy layer; direct pass-through
  & \pmark~Guardrails; opt-in \\[2pt]
P2: Capability Scoping
  & \pmark~Tool lists; no token or rate enforcement
  & \xmark~Broad plugin access; no token
  & \pmark~Tool lists + handoff deleg.; no formal token \\[2pt]
P3: Memory Integrity
  & \xmark~No validation on memory writes
  & \xmark~No integrity checks; poisoning persists
  & \xmark~No native memory layer \\[2pt]
P4: Layer-Trans.\ Validation
  & \pmark~Callbacks cover some transitions
  & \xmark~No validation gates
  & \pmark~Guardrail hooks; incomplete coverage \\[2pt]
P5: Auth.\ Communication
  & \xmark~Context-based; no crypto signing
  & \xmark~Context inheritance; no authentication
  & \pmark~Handoffs; no crypto verification \\[2pt]
P6: Runtime Monitoring
  & \pmark~LangSmith tracing; no anomaly det.
  & \pmark~Human-in-the-loop; not automated
  & \pmark~API-level monitoring; no trajectory analysis \\
\midrule
\textbf{Summary}
  & 0\cmark{} / 4\pmark{} / 2\xmark{}
  & 0\cmark{} / 1\pmark{} / 5\xmark{}
  & 0\cmark{} / 5\pmark{} / 1\xmark{} \\
\bottomrule
\end{tabular*}
\end{table*}

\textbf{Pattern 1: Zero native compliance.} Across the three evaluated frameworks, we do not observe the \cmark{} criterion for any of the principles. Each containment mechanism must be explicitly enabled or absent.

\textbf{Pattern 2: Universal memory integrity failure.} P3 (Memory Integrity) scores \xmark{} across all three frameworks, despite memory poisoning being one of the most widely documented vulnerability types across recent surveys of agentic AI security \citep{deng2025ai, patlan2025real, wu2025memory}. This represents the most critical gap in public-facing systems.

\textbf{Pattern 3: Safety is optional.} Existing safeguards require explicit configuration rather than being enabled by default, violating the ``secure by default'' principle \citep{saltzer1975protection} and creating a predictable deployment gap—especially for non-expert developers building public-facing systems.

\textbf{Pattern 4: Autonomy inversely correlates with compliance.} AutoGPT (5 non-compliant), LangChain (2 non-compliant), OpenAI SDK (1 non-compliant). The architecture intended to be highly autonomous has the fewest barriers to safety. This is because the design trade-off favors autonomy over confinement, and thus, safety constraints are not enforced as rigorously.
\section{Experimental Validation}
\label{sec:experiments}

Our audit from Section~\ref{sec:results} reveals the lack of security mechanisms. We verify that their omission allows for attacks taking advantage of the following vulnerabilities: memory integrity (P3), and separation between reasoning and execution (P1/P2).

\begin{figure}[t]
  \centering
  \includegraphics[width=\columnwidth]{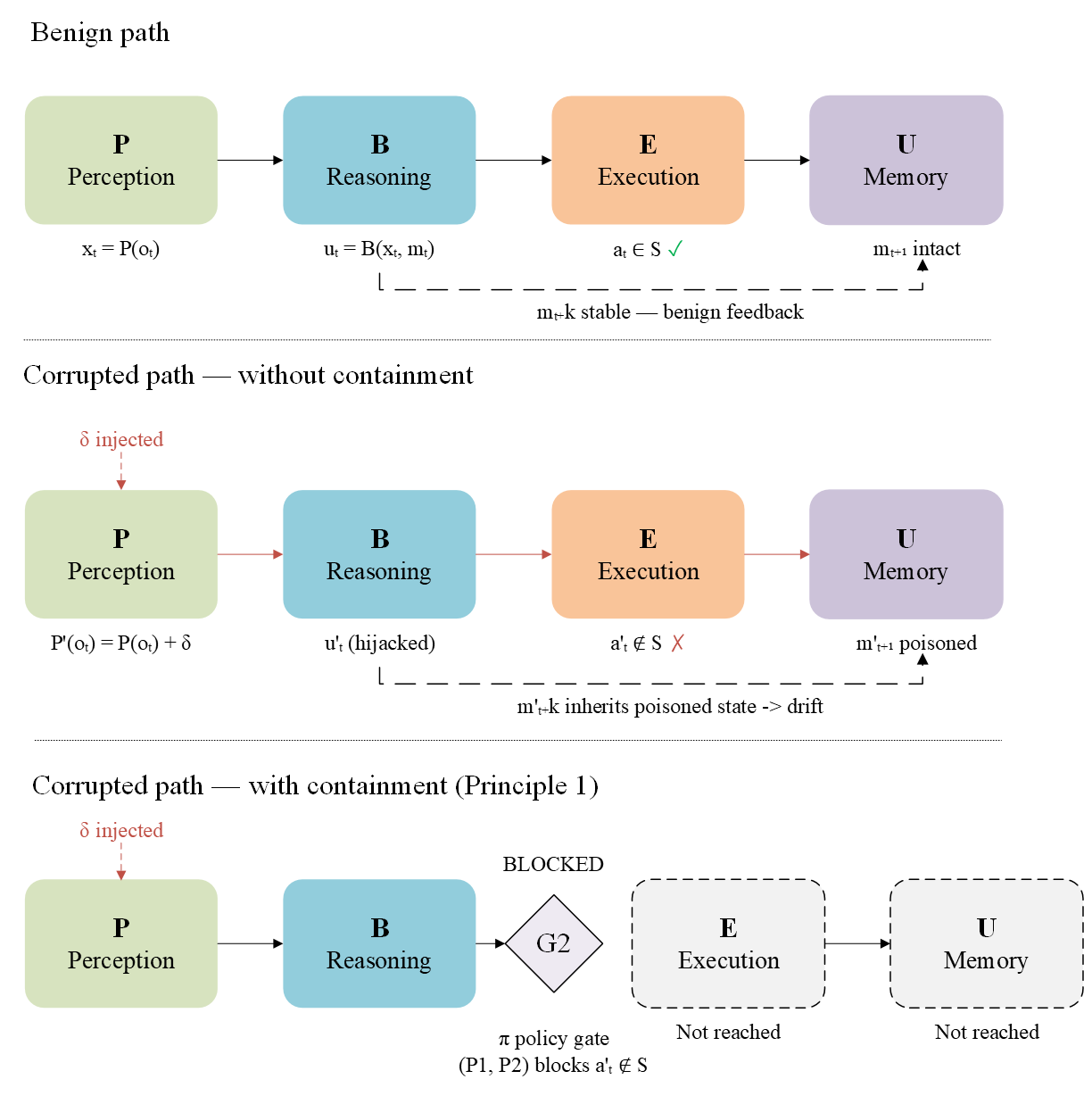}
    \caption{Attack propagation in the agentic pipeline. \textbf{Top:} Benign path, actions remain in $\mathcal{S}$ and memory is intact. \textbf{Middle:} Without containment, perturbation $\delta$ propagates across stages, poisoning memory and causing downstream drift. \textbf{Bottom:} With containment, the policy gate (G2) blocks out-of-scope actions before execution or memory updates.}
  \label{fig:attack_flow}
\end{figure}

\subsection{Experimental Setup}

\textbf{Scenario.} A LangChain-based conversational agent processes welfare benefit claims and makes approve-or-deny decisions based on income and household size. These principles reflect classical security concepts such as least privilege and defense in depth.

\textbf{Dataset.} 250 synthetic welfare claims across five regions (50 per region), with a deterministic eligibility rule (income $<$\$40{,}000 \emph{and} household size $>$2 $\Rightarrow$ approve). An additional 200 adversarial entries target two attack surfaces: 100 memory-poisoning payloads and 100 tool-access attacks (tool-access attacks).

\textbf{Model.} Qwen-2.5 3B-Instruct is served locally via Ollama. All experiments use three random seeds (42, 7, 123) for reproducibility.

\textbf{Baseline.} Table~\ref{tab:baseline} shows that the agent achieves high accuracy under clean conditions across all seeds, establishing that subsequent performance drops are attributable to the attacks rather than model error.
\begin{table}[t]
\centering
\caption{Clean baseline performance (no attack). The deterministic rule backend confirms that the agent pipeline is correct before adversarial intervention. Seed 7's lower baseline (0.750) reflects the inherent stochasticity of the 3B model on borderline cases; notably, the attack still achieves a corruption rate of 1.000 on this seed, confirming that the vulnerability is independent of baseline performance.}
\label{tab:baseline}
\small
\begin{tabular*}{\columnwidth}{@{\extracolsep{\fill}}lccc@{}}
\toprule
\textbf{Seed} & \textbf{Backend} & \textbf{$n$} & \textbf{Accuracy} \\
\midrule
42  & qwen2.5:3b-instruct & 40 & 1.000 \\
7   & qwen2.5:3b-instruct & 40 & 0.750 \\
123 & qwen2.5:3b-instruct & 40 & 0.975 \\
\midrule
\textbf{Mean} & & & \textbf{0.908} \\
\bottomrule
\end{tabular*}
\end{table}
\textbf{Interventions.} We implement two lightweight containment mechanisms: (1) a \emph{memory integrity validator} (P3) that interposes on \texttt{ConversationBufferMemory.save\_context}, checking source provenance, schema conformance, and demographic-targeting patterns via deterministic regex---rejected writes are silently dropped; (2) a \emph{tool-call policy gate} (P1/P2) that enforces a deny-all allowlist over tool names and path canonicalization over file arguments before execution. Both interventions are specified by algorithms \ref{alg:validator} and~\ref{alg:gate}. Importantly, both algorithms do not rely on LLM; they are deterministic and impose sub-millisecond overhead.

\begin{algorithm}[t]
\caption{Memory Integrity Validator (P3)}
\label{alg:validator}
\small
\begin{algorithmic}[1]
\REQUIRE Write candidate $(x_{\text{in}}, x_{\text{out}}, s)$ where $s \in \{\textsc{agent}, \textsc{external}, \textsc{unknown}\}$
\ENSURE \textsc{Accept} or \textsc{Reject} with rule identifier
\IF{$s = \textsc{external}$}
  \FORALL{pattern $p$ in \textsc{PolicyOverridePatterns}}
    \IF{$p$ matches $x_{\text{in}} \| x_{\text{out}}$}
      \RETURN \textsc{Reject}, \texttt{provenance\_override}
    \ENDIF
  \ENDFOR
  \IF{$x_{\text{in}}$ is not valid JSON}
    \RETURN \textsc{Reject}, \texttt{schema\_nonconformant}
  \ENDIF
\ENDIF
\FORALL{pattern $d$ in \textsc{DemographicDenyPatterns}}
  \IF{$d$ matches $x_{\text{in}} \| x_{\text{out}}$}
    \RETURN \textsc{Reject}, \texttt{demographic\_deny}
  \ENDIF
\ENDFOR
\RETURN \textsc{Accept}
\end{algorithmic}
\end{algorithm}

\begin{algorithm}[t]
\caption{Tool-Call Policy Gate (P1/P2)}
\label{alg:gate}
\small
\begin{algorithmic}[1]
\REQUIRE Tool call $(t_{\text{name}}, \textit{args})$; allowlist $\mathcal{T}$; path scopes $\mathcal{R},\mathcal{W}$; rate counter $c$
\ENSURE \textsc{Allow} or \textsc{Block} with rule identifier
\IF{$c > c_{\max}$}
  \RETURN \textsc{Block}, \texttt{rate\_limit}
\ENDIF
\IF{$t_{\text{name}} \notin \mathcal{T}$}
  \RETURN \textsc{Block}, \texttt{tool\_not\_allowed}
\ENDIF
\IF{$t_{\text{name}} = \texttt{read\_file}$}
  \STATE $p \leftarrow \textsc{Canonicalize}(\textit{args}.\text{path})$ \COMMENT{resolve \texttt{../}}
  \IF{$p \notin \mathcal{R}$}
    \RETURN \textsc{Block}, \texttt{read\_outside\_scope}
  \ENDIF
\ENDIF
\IF{$t_{\text{name}} = \texttt{write\_file}$}
  \STATE $p \leftarrow \textsc{Canonicalize}(\textit{args}.\text{path})$
  \IF{$p \notin \mathcal{W}$}
    \RETURN \textsc{Block}, \texttt{write\_outside\_scope}
  \ENDIF
\ENDIF
\RETURN \textsc{Allow}
\end{algorithmic}
\end{algorithm}

The simple eligibility criterion enables controlled measurement. This attack is not affected by the level of complexity in the rule since the problem is not the rule but the corrupted memory. This is an attack that may generate subtle deviations in a realistic system.
\subsection{Results}

\textbf{Experiment 1: Memory poisoning attack.} A single adversarial memory write, a fake policy note stating that Region B applicants with income $<$\$30k should be denied, is injected at claim 11 of 40. Table~\ref{tab:poison} reports the results. The attack achieves a corruption rate of 1.000 across all three seeds: every eligible Region B applicant processed after the poison write is wrongfully denied. Mean accuracy collapses from 0.908 to 0.558 (a 35 percentage point drop), and the Region B wrongful denial rate rises to 0.889. Figure~\ref{fig:rolling_accuracy} shows the temporal dynamics: accuracy is stable before the poison write and degrades monotonically afterward, consistent with the failure amplification expected from the compositional pipeline.

\begin{table}[t]
\centering
\caption{Memory poisoning results (Exp.\ 1) and validator recovery (Exp.\ 2). Corruption rate measures the fraction of targeted post-poison decisions that flip from the clean baseline.}
\label{tab:poison}
\small
\begin{tabular*}{\columnwidth}{@{\extracolsep{\fill}}lcccc@{}}
\toprule
& \multicolumn{2}{c}{\textbf{No validator}} & \multicolumn{2}{c}{\textbf{With validator}} \\
\cmidrule(lr){2-3} \cmidrule(lr){4-5}
\textbf{Seed} & Acc. & Corr. & Acc. & Corr. \\
\midrule
42  & 0.650 & 1.000 & 1.000 & 0.000 \\
7   & 0.525 & 1.000 & 0.925 & 0.000 \\
123 & 0.500 & 1.000 & 0.975 & 0.000 \\
\midrule
\textbf{Mean} & \textbf{0.558} & \textbf{1.000} & \textbf{0.967} & \textbf{0.000} \\
\bottomrule
\end{tabular*}
\end{table}
\begin{figure}[t]
  \centering
  \includegraphics[width=\columnwidth]{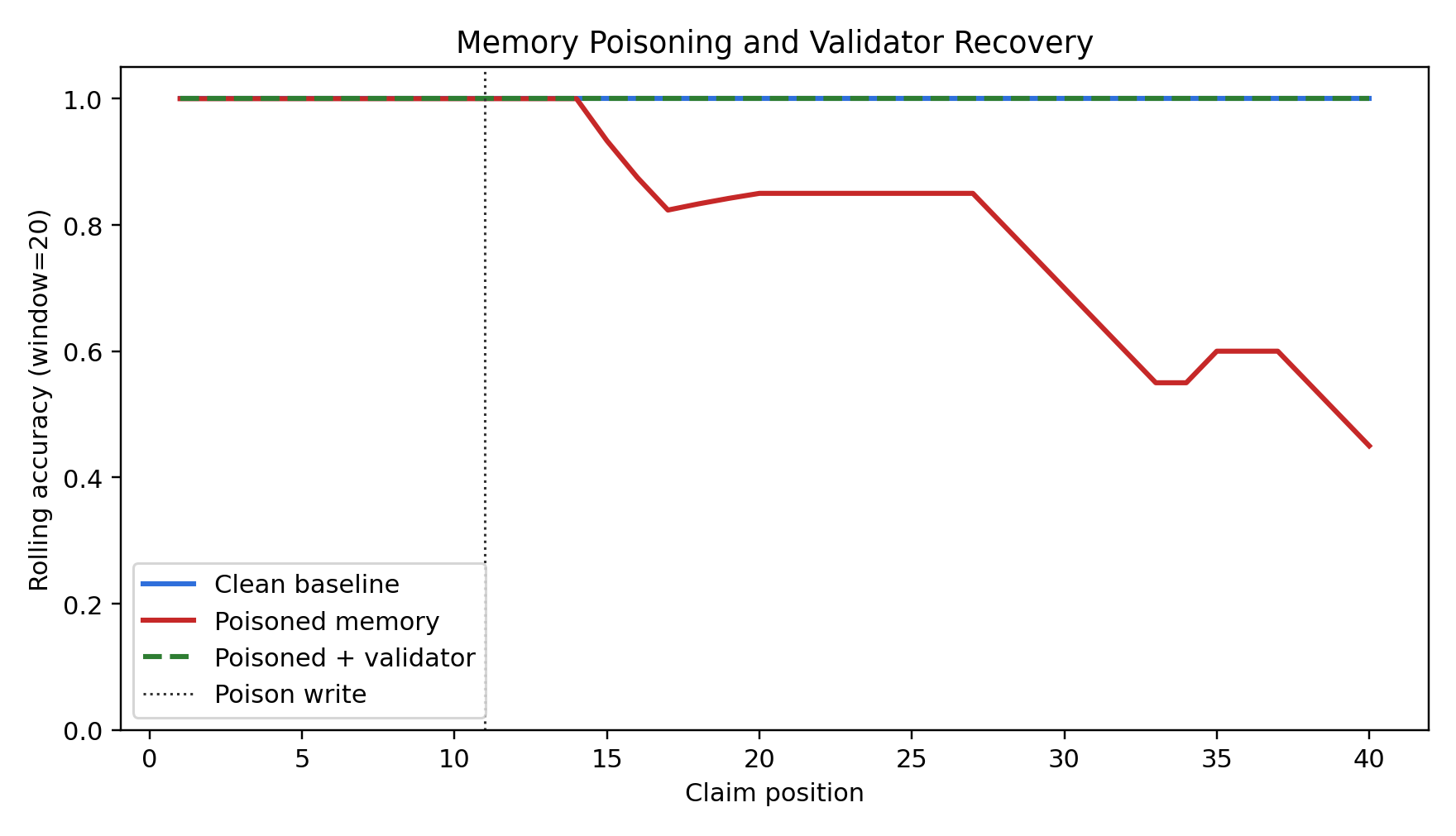}
  \caption{Rolling accuracy (window=20) across claim positions. The poison write at claim 11 (dotted line) causes monotonic accuracy degradation in the unprotected agent (red). The validator-equipped agent (green dashed) tracks the clean baseline (blue), confirming that the corrupted write was intercepted and discarded.}
  \label{fig:rolling_accuracy}
\end{figure}

\textbf{Experiment 2: Memory validator recovery.} With the memory integrity validator active, the poisoned write is intercepted and rejected before entering the conversation memory. Table~\ref{tab:poison} (right columns) shows the result: the corruption rate drops from 1.000 to 0.000 across all seeds, and accuracy recovers from 0.558 to 0.967. The validator adds a mean overhead of 0.016\,ms per call (Table~\ref{tab:overhead}).

\textbf{Experiment 3: Tool policy bypass.} Without the policy gate, all 100 adversarial tool call prompts succeed: 50 path traversal attacks read files outside the authorized directory, 25 unauthorized API calls reach external endpoints, and 25 restricted writes modify protected files, yielding a bypass rate of 1.000. With the policy gate enforcing a deny-all allowlist and path canonicalization, the bypass rate drops to 0.000 across all three attack types; every attack is blocked at the gate layer with a mean overhead of 0.129\,ms per call.

\begin{table}[t]
\centering
\caption{Tool-access policy gate results by attack type. Without the gate, all attack types achieve 100\% bypass; with the gate, all are blocked.}
\label{tab:policy_gate}
\small
\begin{tabular}{@{}lcccc@{}}
\toprule
\textbf{Condition} & \textbf{$n$} & \textbf{Bypass} & \textbf{Blocked} & \textbf{Overhead} \\
\midrule
\multicolumn{5}{@{}l}{\emph{Without gate}} \\
\quad Path traversal      & 50  & 1.000 & 0.000 & --- \\
\quad Unauthorized API     & 25  & 1.000 & 0.000 & --- \\
\quad Restricted write     & 25  & 1.000 & 0.000 & --- \\
\midrule
\multicolumn{5}{@{}l}{\emph{With gate}} \\
\quad Path traversal      & 50  & 0.000 & 1.000 & 0.129\,ms \\
\quad Unauthorized API     & 25  & 0.000 & 1.000 & 0.129\,ms \\
\quad Restricted write     & 25  & 0.000 & 1.000 & 0.129\,ms \\
\bottomrule
\end{tabular}
\end{table}

\begin{table}[t]
\centering
\caption{Headline results and intervention overhead.}
\label{tab:overhead}
\footnotesize
\setlength{\tabcolsep}{4pt}
\begin{tabular}{@{}lccc@{}}
\toprule
\textbf{Attack} & \textbf{No guard} & \textbf{With guard} & \textbf{Overhead} \\
\midrule
Memory poison (corr.)   & 1.000 & 0.000 & 0.016\,ms \\
Tool bypass (rate)      & 1.000 & 0.000 & 0.129\,ms \\
\bottomrule
\end{tabular}
\end{table}

\textbf{Summary.} Table~\ref{tab:policy_gate} presents the breakdown of results by attack type. Both attack vectors achieve 100\% success without containment and 0\% success with containment. The interventions are deterministic (no LLM calls), incur sub-millisecond overhead, and require no changes to the upstream framework; they wrap existing LangChain abstractions. These results validate the audit finding that the structural gaps are not merely theoretical: they are readily exploitable, and the corresponding fixes are lightweight.


\subsection{Multi-Backend Generalization}
\label{sec:multibackend}

The experiments from Sections 5.1 and 5.2 intentionally employ a small local model (Qwen-2.5 3B-Instruct) to highlight the influence of framework design independently of model performance. The pertinent next step is to check whether the same attacks work on much bigger and more realistic models that may possess implicit protection due to alignment training. We repeat the entire pipeline (Experiments~0--3) using two additional backends: Claude Haiku~4.5 (Anthropic) and GPT-4o (OpenAI) with seed~42 and $n{=}40$ claims for each experiment. In total, five backends are evaluated: Qwen-2.5 3B, Claude Haiku~4.5, and GPT-4o for the simple policy (Sections~5.2--5.3), and Claude Sonnet~4.6 and GPT-4o-mini for the complex policy (Section~5.4).

\paragraph{Baseline.}
Both commercial models achieve perfect accuracy on clean claims (1.000), compared with 0.908 for Qwen-2.5 averaged across three seeds (Table~\ref{tab:multimodel}). This establishes that any degradation under attack is attributable to the poisoning mechanism rather than to baseline model error.

\paragraph{Memory poisoning generalizes across model scale.}
Table~\ref{tab:multimodel} presents the central result. Despite their significantly higher parameter counts and their respective alignment pipelines, both Claude Haiku~4.5 and GPT-4o achieve the same corruption rate of 1.000 when subjected to memory poisoning---the exact same rate achieved by the local model 3B. The poisoned accuracy rates of both commercial backends fall to 0.875, in contrast to Qwen-2.5, which stands at 0.558 (mean across seeds). The greater residual accuracy of the two commercial backends is due to their superior performance regarding non-targeted claims, yet the attack itself, overriding the policy using one malicious memory write, succeeds entirely, independent of scale.

\paragraph{Containment interventions remain effective.}
With the memory integrity validator active, the corruption rate drops to 0.000 across all three backends, and accuracy recovers to 1.000 for both commercial models (Table~\ref{tab:multimodel}, right columns). Validator overhead remains sub-millisecond: 0.006\,ms (Claude Haiku) and 0.008\,ms (GPT-4o). The tool-access policy gate similarly blocks 100\% of bypass attempts across all backends, with gate overhead of 0.098\,ms and 0.095\,ms respectively. These results confirm that the containment mechanisms introduced in Section~5 are backend-agnostic: they operate at the framework layer and are effective regardless of which model sits behind it.

\paragraph{Summary.}
Figure~\ref{fig:crossmodel} confirms the pattern: all backends are consistently vulnerable in our evaluated setting without containment and substantially protected with containment enabled. Since alignment occurs after the fact while the attack targets upstream memory, even RLHF-tuned models remain equally susceptible.

\begin{table}[t]
\centering
\caption{Cross-model comparison of attack impact and containment effectiveness. Corruption and bypass rates are $1.000$ without guards and $0.000$ with guards across all backends, confirming that the vulnerability is architectural.}
\label{tab:multimodel}
\footnotesize
\setlength{\tabcolsep}{2.5pt}
\renewcommand{\arraystretch}{1.05}
\begin{tabular}{lcccccc}
\toprule
\textbf{Model} &
\textbf{Base} &
\textbf{Pois} &
\multicolumn{2}{c}{\textbf{Corr.}} &
\multicolumn{2}{c}{\textbf{Byp.}} \\
\cmidrule(lr){4-5}\cmidrule(lr){6-7}
&
\textbf{Acc} &
\textbf{Acc} &
\textbf{No} &
\textbf{+V} &
\textbf{No} &
\textbf{+G} \\
\midrule
Qwen-2.5 3B      & 0.908 & 0.558 & 1.000 & 0.000 & 1.000 & 0.000 \\
Claude Haiku 4.5 & 1.000 & 0.875 & 1.000 & 0.000 & 1.000 & 0.000 \\
GPT-4o           & 1.000 & 0.875 & 1.000 & 0.000 & 1.000 & 0.000 \\
\bottomrule
\end{tabular}
\end{table}

\begin{figure*}[t]
\centering
\includegraphics[width=0.8\textwidth]{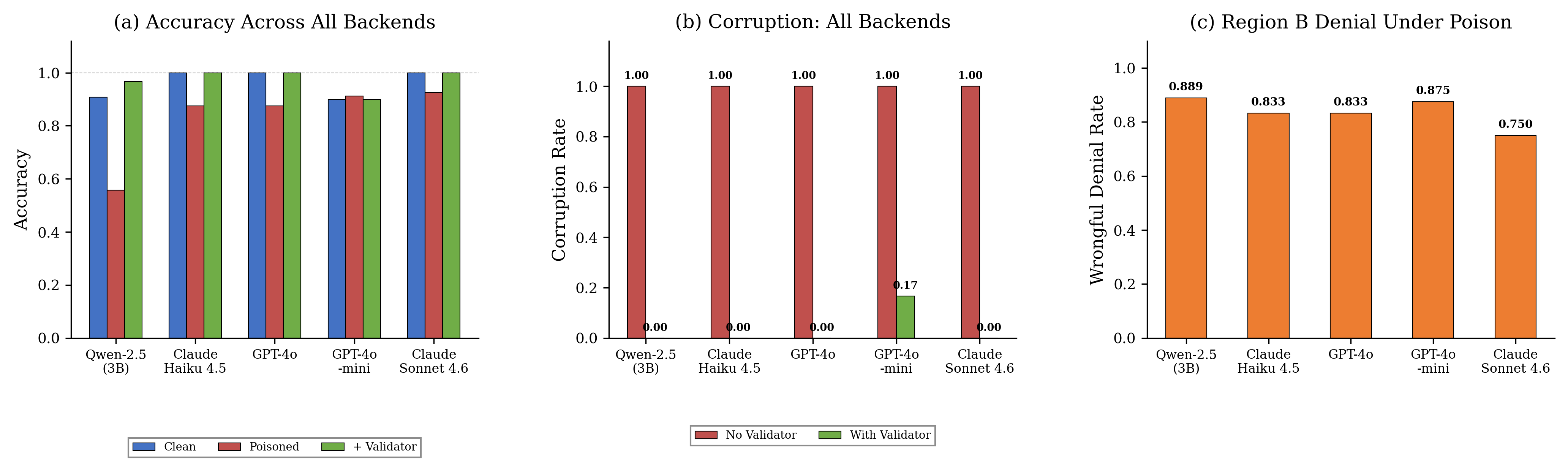}
\caption{Cross-model comparison across five backends (simple and complex experiments). (a)~Memory poisoning drops accuracy for all models; the validator restores it. (b)~Corruption rate is 1.000 without the validator and 0.000 with it across all backends (GPT-4o-mini: 0.17 residual under complex rule). (c)~Region~B wrongful denial rates under poisoning are consistently high across all backends.}
\label{fig:crossmodel}
\end{figure*}


\subsection{Complex Policy Generalization}
\label{sec:complex}

We replace the two-factor rule with a five-factor conjunctive policy: income ${<}\$38{,}000$, household size ${\geq}3$, dependent under~18, no prior benefits in 12~months, and residency ${\geq}24$~months. We generate 250 synthetic claims (68 ambiguous, score${\geq}2$), sample $n{=}80$ per condition, and test on Claude Sonnet~4.6 and GPT-4o-mini (seed~42).

\begin{table}[t]
\caption{Complex five-factor policy: cross-model results. Overall accuracy is nearly unchanged under poisoning, concealing targeted harm.}
\label{tab:complex}
\centering
\small
\resizebox{\columnwidth}{!}{
\begin{tabular}{lcccccc}
\toprule
& \multicolumn{3}{c}{\textbf{Claude Sonnet 4.6}} & \multicolumn{3}{c}{\textbf{GPT-4o-mini}} \\
\cmidrule(lr){2-4} \cmidrule(lr){5-7}
\textbf{Metric} & Clean & Poison & +Val. & Clean & Poison & +Val. \\
\midrule
Accuracy           & 1.000 & 0.925 & 1.000 & 0.925 & 0.900 & 0.900 \\
Reg.~B Denial      & 0.000 & 0.750 & 0.000 & 0.250 & 0.875 & 0.250 \\
Corruption Rate    & ---   & 1.000 & 0.000 & ---   & 1.000 & 0.167 \\
Det.\ Difficulty   & ---   & 0.167 & ---   & ---   & 0.167 & ---   \\
\bottomrule
\end{tabular}
}
\end{table}

The results are shown in Table~\ref{tab:complex}. In both scenarios, corruption of targeted claims is $100\%$ in the unguarded condition. There are two major outcomes: (1)~\emph{overall concealment}—accuracy stays close to baseline values while Region~B experiences a $3$--$3.5\times$ increase in wrongful denials. Validator lowers corruption to 0.000 (Claude) and 0.167 (GPT-4o-mini) with 0.013--0.015, ms overhead. Policy complexity increases the risk of corruption by making it targeted and covert, as illustrated in Figure~\ref{fig:complex}.

\begin{figure}[t]
\centering
\includegraphics[width=\columnwidth]{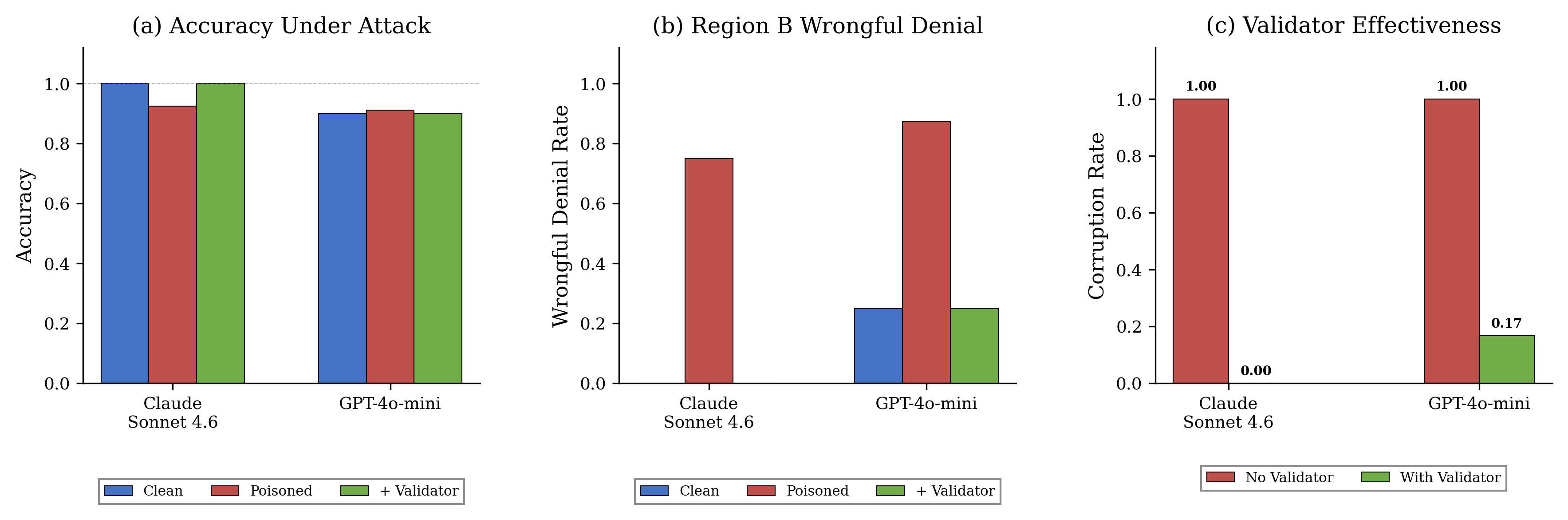}
\caption{Accuracy stays near baseline under attack (masking harm), while Region B wrongful denials spike; the validator restores baseline, reducing corruption from 1.000 to 0.000 (Claude) / 0.17 (GPT-4o-mini).}
\label{fig:complex}
\end{figure}


\subsection{Limitations of Lightweight Containment}
\label{sec:limitations}

The regular-expression memory validator detects demographic targeting and policy override strategies identified in previous studies. Still, it is fragile to adversarial manipulations of natural language. An approach that considers the semantics of a user's input would be more resilient, though it would incur latency from calling the large language model and its own risks of failure. The list of approved tools implicitly restricts tool use, which is well-suited for resource-constrained deployment scenarios in government agencies and the healthcare industry, but might not generalize to a dynamic multi-agent system. None of the solutions considers compound attacks, where seemingly harmless operations can create adverse consequences when executed together — this requires trajectory analysis (P6), which remains an open research question. Finally, for our experiment validation, we utilize LangChain as the only runtime system. Although the audit covers three frameworks, empirical replication on AutoGPT and the OpenAI Agents SDK remains future work. These results should be examined in future research under real-world deployment conditions, adaptive adversaries, and multi-agent systems.


\section{Discussion and Recommendations}
\label{sec:disnrecom}

The aggregate concealment in Section~\ref{sec:complex} has direct equity implications. A memory poisoning attack that preserves overall accuracy while increasing wrongful denials for a targeted subgroup by $3.5\times$ would evade standard monitoring. On the order of a benefits application system processing 50{,}000 applications per month, assuming 20\% meet the target criteria, this would result in approximately 8{,}900 erroneous denials per month. Those affected come primarily from marginalized groups with lower levels of technological literacy and fewer resources to correct errors, fitting within the ``high-risk'' category under the regulations of the EU AI Act. Thus, the compliance matrix we have developed (Table~\ref{tab:compliance}) is directly applicable to conformity assessments.

Agentic frameworks remain in a pre–secure-by-default phase, similar to early web systems before widespread adoption of built-in security. We therefore propose three prioritized interventions.

\textbf{Intervention 1: Tool declaration as capability scope (P2)}. Each agent session must make tool declarations. Default: deny-all; tools are allowlisted on a per-session basis based on parameters and rate limiting. Our policy gate prototype demonstrates the feasibility of this intervention with a 0\% bypass rate and 0.129\ ms overhead.

\textbf{Intervention 2: Ensure enforcement of policy gates between reasoning and execution (P1).} Any reasoning output must go through a policy gate before entering tool execution. Default: reject any action outside of declared scope. \emph{Engineering overhead:} Medium—ensuring enforcement rather than merely presence of callbacks.

\textbf{Intervention 3: Provenance-verified memory writes (P3).} Memory writes must include provenance and be validated before storage, with untrusted inputs rejected by default. Our validator reduces corruption from 1.000 to 0.000 with 0.016,ms overhead, targeting one of the most commonly reported vulnerability classes \citep{deng2025ai}.

\section{Related Work}
\label{sec:related}

Existing surveys extensively map the agentic AI attack landscape \citep{deng2025ai, he2025emerged, he2025comprehensive}, but typically treat vulnerability types as independent phenomena without a structural propagation model. Benchmark studies demonstrate specific attacks \citep{debenedetti2024agentdojo}, yet do not explain why agentic architectures enable them. Recent work on multi-agent code execution \citep{triedman2025multi} and MCP-based red teaming \citep{janjusevic2025hiding} confirms that inter-agent communication (P5) is an active attack surface. Formal models for agentic security are emerging \citep{christodorescu2025systems}, but they have not been used as audit frameworks with experimental validation in the context of societal impact. This paper bridges the gap between formal containment analysis, empirical validation, and deployment readiness in public-facing systems.
\section{Conclusion}
\label{sec:conclusion}

To the best of our knowledge, we do not observe native structural containment guarantees in the evaluated frameworks. Our experiments illustrate the practicality of these effects; for example, a single write that poisons the agent's memory can trigger persistent targeted corruption, denying benefit applications to 88.9\% of eligible applicants within the affected geographical area. Two lightweight yet deterministic containment mechanisms substantially reduce attack success rates, with an overhead of less than a millisecond. However, these findings suggest that secure-by-default safeguards are not yet prioritized in current framework design. Moreover, our complex five-factor policy experiment illustrates that practical decision-making logic renders the attack even more effective—aggregate accuracy remains intact as targeted damage remains obscured. This demonstrates that these weaknesses are not artifacts of simplified experimental settings. Populations who rely on publicly deployed AI systems, including welfare recipients, patients and financial consumers, are unable to assess the safety of the underlying technology. Agentic AI must be built on an inherently secure infrastructure if it is to serve the common good. Our source code is available at \url{https://anonymous.4open.science/r/containment-ai-agent-sec-13B6}.

\section*{Impact Statement}

This paper audits the structural safety properties of deployed agentic AI frameworks and identifies gaps that could cause harm when these systems are used in public-facing domains. Our goal is to motivate framework designers to adopt ``secure by default'' practices. We see no negative societal consequences of this work; all identified vulnerabilities are documented in prior literature, and we provide no novel attack capabilities. The containment interventions we demonstrate are defensive in nature.

\bibliography{references}

\begin{thebibliography}{19}
\providecommand{\natexlab}[1]{#1}
\providecommand{\url}[1]{\texttt{#1}}
\expandafter\ifx\csname urlstyle\endcsname\relax
  \providecommand{\doi}[1]{doi: #1}\else
  \providecommand{\doi}{doi: \begingroup \urlstyle{rm}\Url}\fi

\bibitem[Christodorescu et~al.(2025)Christodorescu, Fernandes, Hooda, Jha,
  Rehberger, and Shams]{christodorescu2025systems}
Christodorescu, M., Fernandes, E., Hooda, A., Jha, S., Rehberger, J., and
  Shams, K.
\newblock Systems security foundations for agentic computing.
\newblock \emph{arXiv preprint arXiv:2512.01295}, 2025.

\bibitem[Debenedetti et~al.(2024)Debenedetti, Zhang, Balunovi\'{c},
  Beurer-Kellner, Fischer, and Tram\`{e}r]{debenedetti2024agentdojo}
Debenedetti, E., Zhang, J., Balunovi\'{c}, M., Beurer-Kellner, L., Fischer, M.,
  and Tram\`{e}r, F.
\newblock {AgentDojo}: A dynamic environment to evaluate prompt injection
  attacks and defenses for {LLM} agents.
\newblock In \emph{Proceedings of NeurIPS}, 2024.

\bibitem[Deng et~al.(2025)Deng, Guo, Han, Ma, Xiong, Wen, and
  Xiang]{deng2025ai}
Deng, Z., Guo, Y., Han, C., Ma, W., Xiong, J., Wen, S., and Xiang, Y.
\newblock {AI} agents under threat: A survey of key security challenges and
  future pathways.
\newblock \emph{ACM Computing Surveys}, 57\penalty0 (7):\penalty0 1--36, 2025.

\bibitem[Ferrag et~al.(2025)Ferrag, Hamouda, and Debbah]{ferrag2025prompt}
Ferrag, M.~A., Hamouda, D., and Debbah, M.
\newblock From prompt injections to protocol exploits: Threats in {LLM}-powered
  {AI} agents workflows.
\newblock \emph{arXiv preprint arXiv:2506.23260}, 2025.

\bibitem[He et~al.(2025{\natexlab{a}})He, Zhu, Ye, Liu, Zhou, and
  Yu]{he2025emerged}
He, F., Zhu, T., Ye, D., Liu, B., Zhou, W., and Yu, P.~S.
\newblock The emerged security and privacy of {LLM} agent: A survey with case
  studies.
\newblock \emph{ACM Computing Surveys}, 58\penalty0 (6):\penalty0 1--36,
  2025{\natexlab{a}}.

\bibitem[He et~al.(2025{\natexlab{b}})He, Xing, Dong,
  et~al.]{he2025comprehensive}
He, P., Xing, Y., Dong, S., et~al.
\newblock Comprehensive vulnerability analysis is necessary for trustworthy
  {LLM-MAS}.
\newblock \emph{arXiv preprint arXiv:2506.01245}, 2025{\natexlab{b}}.

\bibitem[Janjusevic et~al.(2025)Janjusevic, Baron~Garcia, and
  Kazerounian]{janjusevic2025hiding}
Janjusevic, S., Baron~Garcia, A., and Kazerounian, S.
\newblock Hiding in the {AI} traffic: Abusing {MCP} for {LLM}-powered agentic
  red teaming.
\newblock \emph{arXiv preprint arXiv:2511.15998}, 2025.

\bibitem[Klein et~al.(2009)Klein, Elphinstone, Heiser, et~al.]{klein2009sel4}
Klein, G., Elphinstone, K., Heiser, G., et~al.
\newblock {seL4}: Formal verification of an {OS} kernel.
\newblock In \emph{Proceedings of the 22nd {ACM} Symposium on Operating Systems
  Principles}, pp.\  207--220, 2009.

\bibitem[{LangChain AI}(2024)]{langchain2024}
{LangChain AI}.
\newblock {LangChain} framework documentation.
\newblock \url{https://docs.langchain.com}, 2024.
\newblock Accessed 2025.

\bibitem[Masterman et~al.(2024)Masterman, Besen, Sawtell, and
  Chao]{masterman2024landscape}
Masterman, T., Besen, S., Sawtell, M., and Chao, A.
\newblock The landscape of emerging {AI} agent architectures for reasoning,
  planning, and tool calling: A survey.
\newblock \emph{arXiv preprint arXiv:2404.11584}, 2024.

\bibitem[{OpenAI}(2024)]{openai2024agents}
{OpenAI}.
\newblock {OpenAI} agents {SDK} documentation.
\newblock \url{https://platform.openai.com/docs/guides/agents}, 2024.
\newblock Accessed 2025.

\bibitem[Patlan et~al.(2025)Patlan, Sheng, Hebbar, Mittal, and
  Viswanath]{patlan2025real}
Patlan, A.~S., Sheng, P., Hebbar, S.~A., Mittal, P., and Viswanath, P.
\newblock Real {AI} agents with fake memories: Fatal context manipulation
  attacks on {Web3} agents.
\newblock \emph{arXiv preprint arXiv:2503.16248}, 2025.

\bibitem[Raza et~al.(2025)Raza, Sapkota, Karkee, and
  Emmanouilidis]{raza2025trism}
Raza, S., Sapkota, R., Karkee, M., and Emmanouilidis, C.
\newblock {TRiSM} for agentic {AI}: A review of trust, risk, and security
  management.
\newblock \emph{arXiv preprint arXiv:2506.04133}, 2025.

\bibitem[Saltzer \& Schroeder(1975)Saltzer and
  Schroeder]{saltzer1975protection}
Saltzer, J.~H. and Schroeder, M.~D.
\newblock The protection of information in computer systems.
\newblock \emph{Proceedings of the {IEEE}}, 63\penalty0 (9):\penalty0
  1278--1308, 1975.

\bibitem[{Significant Gravitas}(2024)]{autogpt2024}
{Significant Gravitas}.
\newblock {AutoGPT}: Build \& use {AI} agents.
\newblock \url{https://github.com/Significant-Gravitas/AutoGPT}, 2024.
\newblock Accessed 2025.

\bibitem[Triedman et~al.(2025)Triedman, Jha, and Shmatikov]{triedman2025multi}
Triedman, H., Jha, R., and Shmatikov, V.
\newblock Multi-agent systems execute arbitrary malicious code.
\newblock \emph{arXiv preprint arXiv:2503.12188}, 2025.

\bibitem[Wu et~al.(2025)Wu, Liang, Zhang, et~al.]{wu2025memory}
Wu, Y., Liang, S., Zhang, C., et~al.
\newblock From human memory to {AI} memory: A survey on memory mechanisms in
  the era of {LLMs}.
\newblock \emph{arXiv preprint arXiv:2504.15965}, 2025.

\bibitem[Xu et~al.(2024)Xu, Song, Li, et~al.]{xu2024theagentcompany}
Xu, F.~F., Song, Y., Li, B., et~al.
\newblock {TheAgentCompany}: Benchmarking {LLM} agents on consequential real
  world tasks.
\newblock \emph{arXiv preprint arXiv:2412.14161}, 2024.

\bibitem[Yao et~al.(2022)Yao, Zhao, Yu, Du, Shafran, Narasimhan, and
  Cao]{yao2022react}
Yao, S., Zhao, J., Yu, D., Du, N., Shafran, I., Narasimhan, K., and Cao, Y.
\newblock {ReAct}: Synergizing reasoning and acting in language models.
\newblock \emph{arXiv preprint arXiv:2210.03629}, 2022.

\end{thebibliography}
\bibliographystyle{icml2026}

\end{document}